\documentclass[letterpaper, journal, twoside]{IEEEtran}

\usepackage{amsmath,amssymb,amsfonts}
\usepackage{algorithm}
\usepackage{algpseudocode}
\usepackage{bm}
\usepackage{textcomp}
\usepackage{stfloats}
\usepackage{url}
\usepackage{verbatim}
\usepackage{graphicx}
\usepackage{siunitx}
\usepackage{balance}
\usepackage{subcaption}
\usepackage{cite}

\usepackage{amsthm}
\newtheorem{theorem}{Theorem}
\newtheorem{remark}{Remark}

\newtheorem{definition}{Definition}[section]
\newtheorem{assumption}{Assumption}

\DeclareMathOperator*{\diag}{diag}

\usepackage{lipsum}
\usepackage{mathtools}

\DeclareMathOperator*{\argmax}{arg\,max}

\usepackage[pdftex, plainpages=false,hypertexnames=true,pdfnewwindow=true,colorlinks=true,citecolor=blue,linkcolor=red,urlcolor=blue,filecolor=blue]{hyperref}%

\begin{document}

\title{ Decentralized Reinforcement Learning for Multi-Agent Multi-Resource Allocation via Dynamic Cluster Agreements }

\author{Antonio Marino, Esteban Restrepo, Claudio Pacchierotti, Paolo Robuffo Giordano
\thanks{Manuscript received: February 23, 2025; Revised: May 6, 2025; Accepted: June 2, 2025.}
\thanks{This paper was recommended for publication by Editor M. Ani Hsieh upon evaluation of the Associate Editor and Reviewers’ comments.}
\thanks{A. Marino is with Univ Rennes, CNRS, Inria, IRISA -- Rennes, France. E-mail: antonio.marino@irisa.fr.}
\thanks{E. Restrepo, C. Pacchierotti and P. Robuffo Giordano are with CNRS, Univ Rennes, Inria, IRISA -- Rennes, France. E-mail: \{esteban.restrepo, claudio.pacchierotti,prg\}@irisa.fr.}
\thanks{This work was supported by the ANR-20-CHIA-0017 project ``MULTISHARED''}
\thanks{Digital Object Identifier (DOI): }
}

\maketitle
\begin{abstract}
This paper addresses the challenge of allocating heterogeneous resources among multiple agents in a decentralized manner. Our proposed method, Liquid-Graph-Time Clustering-IPPO, builds upon Independent Proximal Policy Optimization (IPPO) by integrating dynamic cluster consensus, a mechanism that allows agents to form and adapt local sub-teams based on resource demands. This decentralized coordination strategy reduces reliance on global information and enhances scalability. We evaluate LGTC-IPPO against standard multi-agent reinforcement learning baselines and a centralized expert solution across a range of team sizes and resource distributions. Experimental results demonstrate that LGTC-IPPO achieves more stable rewards, better coordination, and robust performance even as the number of agents or resource types increases. Additionally, we illustrate how dynamic clustering enables agents to reallocate resources efficiently also for scenarios with discharging resources.  
\end{abstract}


\section{INTRODUCTION}
Resource allocation in multi-agent systems (MAS) is a critical challenge in domains such as robotics, logistics, and disaster management. It requires agents to collaboratively address dynamic, heterogeneous demands under environmental and operational constraints. Complexity arises from partial observability, resource heterogeneity, and the need for decentralized decision-making to ensure scalability and robustness~\cite{khamis2015multi, park2021multi, fu2022robust}. For example, in search and rescue, robots must deliver various resources, such as first aid kits or connectivity, to multiple targets. Centralized methods often become impractical in large-scale, real-world scenarios due to computational and communication limits~\cite{doostmohammadian2025survey}.

Several studies have proposed distributed solutions to the multi-robot task allocation problem. Morgan et al.~\cite{morgan2016swarm} employed a distributed auction algorithm to allocate agents to predefined locations, while Goarin et al.~\cite{goarin2024graph} used graph neural networks (GNNs) to reduce per-agent information requirements. Griffith et al.~\cite{griffith2017automated} formulated dynamic allocation as a Markov Decision Process (MDP), demonstrating scalability but without addressing heterogeneity. More recently, Dai et al.~\cite{dai2025heterogeneous} tackled heterogeneous multi-robot task allocation via a decentralized sequential decision-making framework, using an attention-based neural network and a constrained flashforward mechanism to avoid deadlocks.

Other approaches, such as~\cite{camisa2022multi}, addressed heterogeneous allocation through distributed optimization with locally defined costs, though general resource allocation typically involves non-decomposable global objectives. Coffey et al.~\cite{coffey2023covering} improved this by modeling resource distribution as a coverage task with dynamic densities but faced challenges with local minima and numerical instability. Finally, Zhang et al.~\cite{zhang2024opinion} modeled allocation as an opinion dynamics process, enabling decentralized team formation but requiring resets when environments change, limiting adaptability to dynamic settings.

Dynamic clustering agreement in networked systems has been widely studied in recent years. Given a team of agents in a connected graph, cluster consensus seeks to achieve agreement within subgroups of agents smaller than the overall team. Conditions for cluster consensus, such as specific graph topologies, negative edge weights, or non-linear and heterogeneous interaction protocols, are outlined in~\cite{sorrentino2016complete, xia2011clustering}. For instance,~\cite{zhang2024opinion} uses non-linear synchronization dynamics, while~\cite{bizyaeva2022nonlinear} combines nonlinear saturating dynamics with heterogeneity to induce bifurcations that result in clustering. Yet, developing a general methodology to design such non-linear and heterogeneous dynamics remains an open challenge.

In this paper, we propose to learn opinion dynamics using multi-agent reinforcement learning (MARL). MARL enables agents to learn cooperative behaviors without centralized control, and frameworks like decentralized partially observable Markov decision processes (Dec-POMDPs)~\cite{oliehoek2016concise} provide a formal foundation. However, practical challenges such as non-stationarity, scalability, and credit assignment persist~\cite{gronauer2022multi}.

Many MARL methods adopt the centralized training with decentralized execution (CTDE) paradigm~\cite{foerster2018counterfactual, lowe2017multi}, where a global value function guides individual policies. While effective in some settings, CTDE often relies on a well-defined global state and struggles with credit assignment and local coordination. Methods like VDN~\cite{sunehag2018value} and QMIX~\cite{rashid2020monotonic} decompose the global value but falter under partial observability, missing fine-grained coordination. Counterfactual methods~\cite{foerster2018counterfactual} aim to isolate each agent’s impact but can suffer from high variance and unstable training. Multi-objective MARL~\cite{felten2024momaland} is also relevant for multi-resource allocation, where the challenge lies in balancing trade-offs to distribute agents effectively and meet global resource demands.
To address these challenges, we propose a novel decentralized reinforcement learning framework that leverages cluster consensus, specifically tailored for multi-agent, multi-resource allocation problems. Decentralized training has been shown to reduce bias and, in some cases, lower variance in value function estimation~\cite{lyu2021contrasting}, making it a promising approach for tackling the complexities of such environments. Specifically, our contributions are as follows:
\begin{itemize}
    \item Decentralized Reinforcement Learning (RL) with Dynamic Cluster Consensus: We introduce a consensus cluster value function that leverages shared information among agents to address the credit assignment problem effectively. This approach ensures that agents dynamically form clusters to handle subgroup-specific demands while maintaining global coordination.
    \item Reward Design for Multi-Agent Resource Allocation: We propose a hybrid reward structure that balances global incentives, such as reducing overall resource demand, with local rewards that penalize collisions and reward subgroup cooperation. This design bridges the gap between localized decision-making and system-wide optimization.
    \item Experimental Validation in Simulation and Hardware: We validate our approach through extensive simulations and real-world experiments using drones, demonstrating its robustness and efficiency. The results show that our method outperforms state-of-the-art algorithms, such as VDN, QMIX, and Multi-Objective Multi-Agent PPO (MOMAPPO) [8], in terms of resource allocation performance and adaptability.
\end{itemize}


\section{PROBLEM STATEMENT}
\label{sec:prob-statement}

In this section, we introduce the problem using the Markov game formalism, following the reinforcement learning (RL) literature. Consider a resources-assignment task involving a fully cooperative multi-agent team \( \mathcal{V} \) composed of \( N \) agents operating in a bounded, convex environment \( Q \subset \mathbb{R}^{n} \). The task requires delivering \( r \) different resources of different types \(\kappa\) to \( M \) consumers located in \( Q \), where each consumer may demand a subset of the available resources. Each consumer has an interaction area \( R_m \), within which an agent's resource can be released to satisfy the demand, see Figure~\ref{fig:assingment-example}. We consider two types of resources based on their depletion dynamics: persistent resources, which require agents to remain within the release area indefinitely to sustain delivery, and instantaneous resources, which are fully delivered upon the agent's arrival in the interaction area. These two types represent opposite ends of the depletion spectrum, ranging from extremely slow (continuous presence required) to extremely fast (immediate release). 

The system's state \( s \in \mathcal{S} \) is defined by the agents' positions \( p \in Q \), the consumers' positions \( d \in Q \), and the resources associated with the agents' supplies \( \mathrm{so}_i \in \mathbb{R}^r\) for agent $i$ and the consumers' demands \( \mathrm{de}_m \in \mathbb{R}^r\) for the consumer $m$. The state evolves according to a transition probability \( P(s'|s, u): \mathcal{S} \times \mathcal{U} \times \mathcal{S} \rightarrow [0,1] \), where \( u \in \mathcal{U}^N \) represents the agents' joint actions. Agents have access to partial observations \( z_i \in \mathcal{Z} \), which include only the agent's position \( p_i \), their resources \( \mathrm{so}_{i} \), and all consumer positions \( d \) and their associated resource demands \( \mathrm{de}_{m} \). Additionally, agents are allowed to communicate with their neighbours \( \mathcal{N}_i = \{ v_j \in \mathcal{V} \mid \| p_i - p_j \|_2 \leq C, j \neq i \} \), meaning that the communication of agent $v_i$ is restricted to a sphere of radius \( C \) centered at each agent.

The objective of the task is to find a collaborative control policy that enables the agents to satisfy, at best, the resource demands within the environment. To guarantee the existence of a unique resource allocation solution, we assume that the sum of the available supplies $\mathrm{so}_i, \forall i =1,\dots,N$  is less or equal to the sum of $\mathrm{de}_m, \forall m=1,\dots,M$. Moreover, this choice accounts for long-term mission where the robot resources may not be enough to finish the team tasks. Each agent receives a reward \( \mathrm{rw}_{i}(s, u): \mathcal{S} \times \mathcal{U} \rightarrow \mathbb{R} \) for fulfilling demands. The problem can be formally defined as a Decentralized Partially Observable Markov Decision Process (Dec-POMDP) \cite{oliehoek2016concise}, represented by the tuple $\langle \mathcal{S}, \mathcal{U}, P, \{ \mathrm{rw}_{0}, \dots, \mathrm{rw}_{N} \}, \mathcal{Z}, \gamma \rangle,
$ where \( \gamma \in [0,1) \) is the discount factor. Under the agent policy distribution \( \pi_i(z_i) \), the team performance is defined by the discounted objective function:

\begin{equation}
\label{eq:objective}
    J_i = \mathbb{E}_{s_o \sim P_0, s' \sim P, u \sim \pi} \bigg[ \sum_{t=0}^{\infty} \gamma^t \mathrm{rw}_{i} \bigg]; \quad
    J = \sum_{i=1}^{N} J_i.
\end{equation}

The problem is thus reduced to finding the agent reward \( \mathrm{rw}_i \) and the optimal policy distribution \( \pi_i(z_i) \) that maximize \( J \). Furthermore, the reward \( \mathrm{rw}_i \) must be carefully designed to account for collective achievements, such as demand satisfaction, while also considering agent motion constraints (e.g., avoiding collisions, staying within the environment boundaries, etc.). Note that the reward $\mathrm{rw}_i$ depends on the joint state and actions of all agents, not just on agent $i$'s state and control.


\section{METHOD}
\begin{figure*}[t!]
    \begin{subfigure}{0.33\linewidth}
    \centering
    \includegraphics[width=\linewidth]{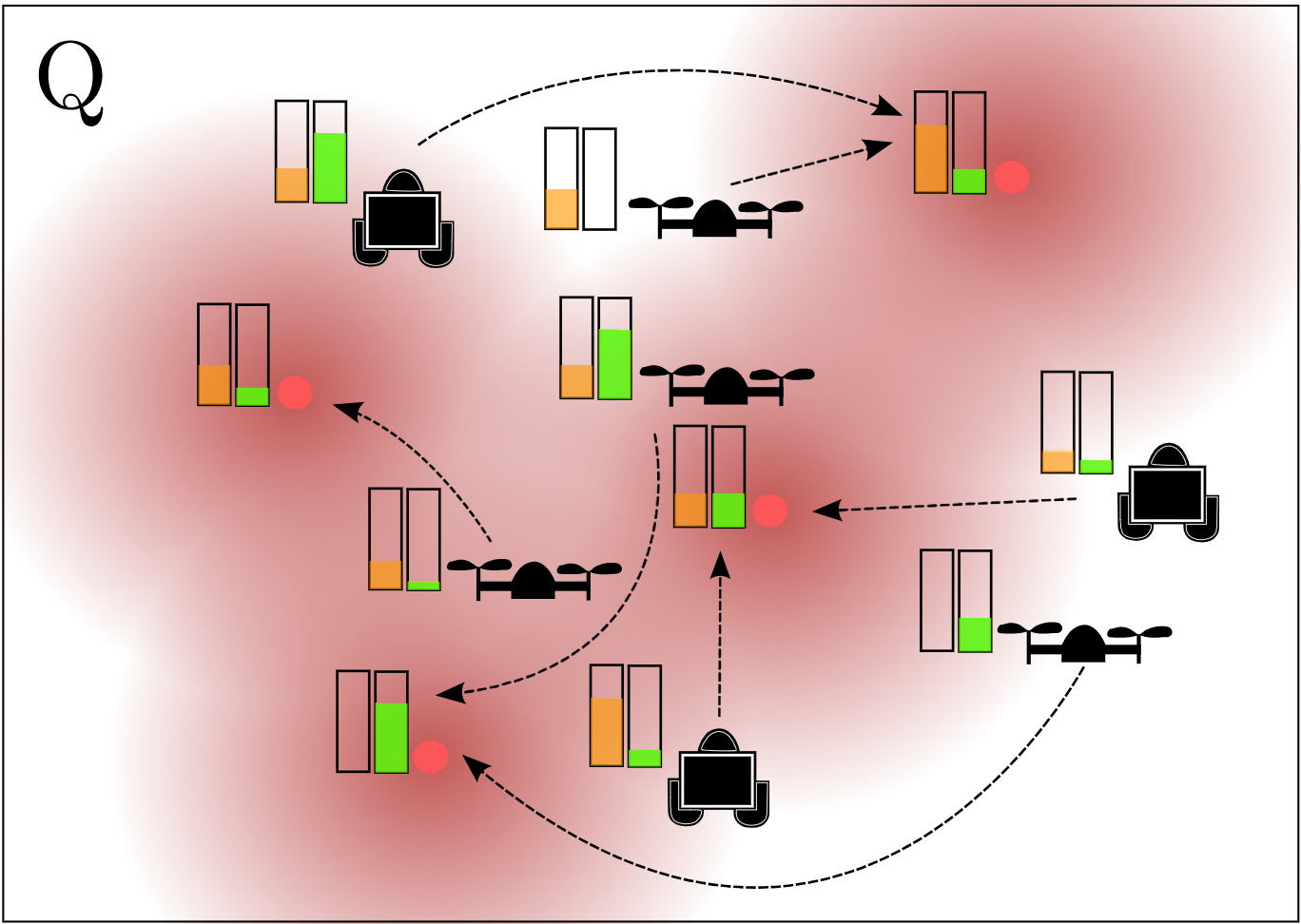}
    \caption{}
    \label{fig:assingment-example}
    \end{subfigure}
    \begin{subfigure}{0.66\linewidth}
    \centering    
    \includegraphics[width=\linewidth]{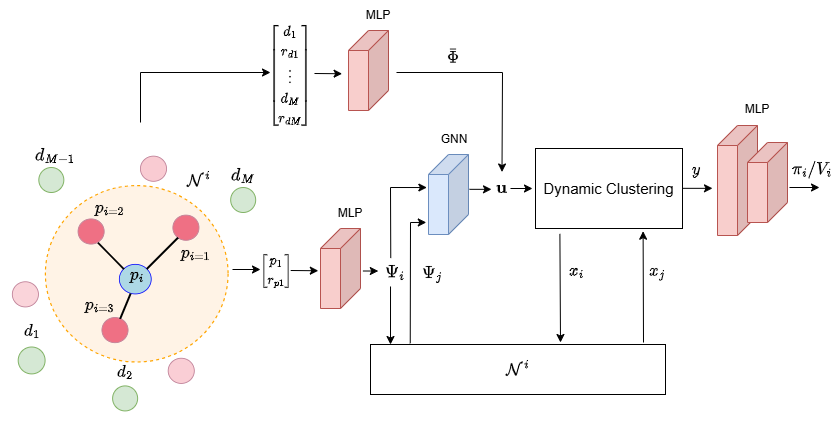}
    \caption{}
    \label{fig:enter-label}
    \end{subfigure}
    \caption{(a) An assignment example involving a group of heterogeneous robots transporting heterogeneous resources shown as bars of different colors. The robots are allocated to optimally fulfill consumer demands located in the red dots. (b) Neural network model architecture for the value and policy estimation}
\end{figure*}

To address the problem described in Section~\ref{sec:prob-statement}, we implement a MARL strategy. Specifically, we consider that agents must collectively supply resources to fully meet consumer demands, as illustrated in Figure~\ref{fig:assingment-example}.

\subsection{Reward shaping}
\label{sec:reward-shaping}
While the literature on fully cooperative games in RL often assumes global rewards equally shared among agents~\cite{gronauer2022multi}, we design a combination of local and global rewards. This task requires a balance between global coordination and localized cooperation among smaller groups of agents contributing to resource delivery for the same consumer. The overall agent reward is the sum of local and global rewards
\begin{equation}
    \mathrm{rw}_i = \mathrm{rw}_{id} + \mathrm{rw}_{im} + \mathrm{rw}_{is} + \mathrm{rw}_{rc} + \sum_{j=1}^{N} \mathrm{rw}_{ij} + \mathrm{rw}_s + \mathrm{rw}_g.
\end{equation}
In the following, we describe each reward element. First, we reward the entire team for reducing the $r$ resources for the $M$ demands at time $t$:

\begin{equation*}
    \mathrm{rw}_g = \sum_{l=1}^{r} \sum_{m=1}^M  \big( \mathrm{de}_{ml}(t-1) - \mathrm{de}_{ml}(t) \big).
\end{equation*}

Where $\mathrm{de}_{ml}$ is $l$-th resource of demanding location $m$. Additionally, agents receive a global reward if there exists at least one agent releasing resources at every consumer location: $\mathrm{rw}_s >0$. For the local reward, we penalize collisions among agents by introducing a negative reward proportional to the squared distance between agents closer than a threshold $\epsilon<R$:

\begin{equation*}
    \begin{aligned}
    \mathrm{rw}_{ij} = -\gamma_{ij} \|p_i - p_j\|^2_2, 
     \quad & \text{with} \quad j \neq i, v_j \in \mathcal{N}_i \\ \gamma_{ij}>0 \quad \text{if} \quad \|p_i - p_j\|^2_2 < \epsilon, \quad &  \gamma_{ij}=0 \quad \text{otherwise} 
    \end{aligned}
\end{equation*}

Agents are rewarded for releasing to the instantaneous depletion resources, with the reward for the agent proportional to the resources released: $\mathrm{rw}_{im} >0$. When the instantaneous consumer demands are fully satisfied, the sub-team contributing to this fulfillment receives an additional reward: $\mathrm{rw}_{rc}>0$. When an agent satisfies a persistent demand, the agent is rewarded with a fixed quantity: $\mathrm{rw}_{is}>0$. We define a reward to guide agent-to-consumer assignments by minimizing demand mismatch \( l_1 = \left\| \mathrm{de} - a^\top \mathrm{so} \right\|_2^2 \), where \( a \in \mathbb{B}^{N \times M} \) is the assignment matrix.To encourage agents clustering and delivering to to proximal resources, we add \mbox{\( l_2 = \| \big( d_{pd} \circ \left( \sum_{l=1}^{r} \mathrm{so}_{l} \otimes \mathbf{1}_M \right) + \mathrm{so} \big( \frac{1}{\mathrm{de}} \big)^\top \big) \circ a\|_1\) }, where \( d_{pd} \in \mathbb{R}^{N \times M} \) is the agent-to-demand distance matrix, \( \frac{1}{\mathrm{de}_m} \) is applied element-wise, and \( \mathbf{1}_M \) is the all-ones vector in \( \mathbb{R}^M \). The optimal \( a \) is found via mixed-integer quadratic programming (MIQP):
\begin{equation}
    \begin{aligned}
         \min\limits_{a} \qquad l_1 + l_2, & \\
         \text{s.t.} \quad \sum_{j=1}^M a_{ij} = 1, \quad \sum_{i=1}^N a_{ij} \geq 1, \quad &
            a \in \{0,1\}^{N \times M}.
    \end{aligned}
    \label{eq:QP}
\end{equation}
The resulting assignments from the optimization ($a_i$) are used to design agent rewards, encouraging agents to align with the correct assignment:
\begin{equation*}
    \mathrm{rw}_{id} = \begin{cases}
        \| p_i(t-1) - d_{a_i} \|_2^2 - \| p_i(t) - d_{a_i} \|_2^2, \hspace{0.2ex} \text{if } p_i(t) \notin R_{a_i}, \\
        \gamma_{id} >0, \qquad \qquad \qquad \qquad \qquad \qquad \text{if } p_i(t) \in R_{a_i}
    \end{cases}
\end{equation*}

Here, $R_{a_i}$ denotes the consumer interaction area. The reward gains are chosen empirically to balance the different contributions. Despite relying on a centralized reward, the policy in Section III.B is fully decentralized, using only local observations and nearby communication. Local rewards, the group reward $\mathrm{rw}_{rc}$, and the accumulated reward $\gamma_{id}$ align agents’ incentives only when targeting the same demanding consumer. CTDE methods like MADDPG~\cite{lowe2017multi} and COMA~\cite{foerster2018counterfactual} are unsuitable here due to difficulties with credit assignment and scalability issues~\cite{lyu2021contrasting}. Similarly, VDN~\cite{sunehag2018value} and QMIX~\cite{rashid2020monotonic} fail to capture subgroup rewards, as shown in Section~\ref{sec:evaluations}. We instead adopt decentralized training and execution (DTDE), with agents estimating individual value functions based on their own rewards. To accelerate training, we use a shared neural network architecture.
\subsection{Neural network architecture}
\label{sec:net-arch}
We deploy two neural networks with the same architecture for the policy and the value function predictions. Given the multi-demand nature of the task, our neural model solution relies upon dynamic cluster consensus that naturally imposes subgroup agreements. Specifically, the demanding consumer features shape the team cluster consensus. In this way, the cluster formation can have a physical meaning for the policy (the agents divide in sub-groups to satisfy the demands) but can also represent the value function clustering, in line with the rewards definition.  The current models~\cite{lian2024distributed, bizyaeva2022nonlinear, zhang2025generalized} developed for dynamic clustering either do not allow the dynamic change of the clustering equilibrium or rely on specific graph topology. Therefore, we design a dynamic model that overcomes these limitations. 

Firstly, each agent processes the resources for each demanding locations $m$ using DeepSets~\cite{zaheer2017deep} as follows: 

\begin{equation}
    \Phi_m = \phi_2\left(\sum_{l = 1}^{r} \phi_1(d_m, \mathrm{de}_{ml},\kappa_{ml})\right).
\end{equation} 
Here, $\phi_1$ and $\phi_2$ are two multi-layer perceptrons (MLPs) that process and combine the consumer's resource information. Similarly, each agent processes its own resources as  $\Psi_i(p_i, \mathrm{so}_{pi},\kappa_{pi})$. The consumer features form the feature vector $\bar{\Phi} = [ \Phi_0, \dots, \Phi_M ]^T \in \mathbb{R}^{M \times G}$, while the agent features constitute the vector $\bar{\Psi} = [ \Psi_0, \dots, \Psi_N ]^T \in \mathbb{R}^{N \times G}$. The agent feature vector $\Psi_i$ is shared among agents using a 1-length graph filter which, together with the consumers features, forms the input $\Delta$ as follows:

\begin{equation}
\Delta = \sum_{k=0}^{1} S^k\bar{\Psi}\hat{B}_k + \sum_{m=1}^{M} \Phi_m  + b_u.
\end{equation}

Where $S \in \mathbb{R}^{N \times N}$ represents the graph's sparsity pattern (e.g., Laplacian or adjacency matrix), $\hat{B}_k, b_u$ are learnable parameters.

\begin{equation}
    \begin{cases}
        \Xi = \text{softmax}\left( \sigma_c \left(\Delta +   \sum\limits_{k=0}^{1} S^kx\hat{A}_k   + b_x \right) \bar{\Phi}^T \right) \\
        f = \rho\left( W \Phi_{\argmax\limits_{m \in \{1, \dots, M \}}\Xi_{m}} + b\right) \\
        \dot{x} = -\left(\tau + f\right) \circ x - S x A + f \circ B_f \\	
        y = [ x,  \Xi \Phi , u, p]
    \end{cases}    .
    \label{eq:model-dynamics}
\end{equation}

Here, $\rho$ denotes the ReLU activation function, while $\sigma_c$ is the Tanh activation function. The team state $x \in \mathcal{X} \subseteq \mathbb{R}^{N \times F}$ has $F$ features. The matrices weights $W, \hat{A}_k, A$ and biases $b, b_x, \tau, B_f$ are learnable parameters. The biases are defined as $\bm{1}_N \otimes \textit{b}$, ensuring uniform biases for all agents. The attention coefficients $\Xi \in \mathbb{R}^{N \times M}$ encode the relative importance of consumer features in guiding agent motion, with the maximum coefficient selecting the dominant consumer features that influence the state $x$. The attention $\Xi$ depends on the current input but also on the clustering states of the entire team. The output $y$ encapsulates the current state, the weighted sum of the consumers' features, input, and agent positions.

We denote the induced infinity norm as $||\cdot||_{\infty}$ and the induced infinite log-norm $\mu_{\infty}(\bm{X}) = \max_{i}(x_{ii}+\sum_{j=1,j\neq 1}^{n}|x_{ij}|)$. In the following, we used the vector operator $\bm{X}_{|}$ that rearranges the elements of matrix $\bm{X}$ in a vector. For the filters in system~\eqref{eq:model-dynamics}, we use the following notation:

\begin{equation}
    S_{I} \triangleq [I, S] \quad
    \hat{A}_{0,K} \triangleq [\hat{A}_0, \hat{A}_1]^T  \quad
    \hat{B}_{0,1} \triangleq [\hat{B}_0, \hat{B}_1 ]^T
    \label{eq:definitions}
\end{equation}

Given the following assumptions

\begin{assumption}
	The bias $B_f$ is unity-bounded: $B_f \subseteq [-1,1]^{N \times F}$ , i.e. $||B_f||_{\infty} \leq 1$,
	\label{assumption1}
\end{assumption}

\begin{assumption}
	Given any two support matrices $||S_{1}(t)||_{\infty}$ and $||S_{2}(t)||_{\infty}$, $\forall t \in \mathbb{R}^{+}$ associated with two different graphs, they are bounded by the same $||\bar{S}||_{\infty}$; moreover, they are lower bounded by $||\tilde{S}||_{\infty}$,
    \label{assumption2}
\end{assumption}
\noindent we provide the following theorem for the neural ODE~\eqref{eq:model-dynamics}.
\begin{theorem}
	\label{diss-lgtc}
	Under Assumption~\ref{assumption1} and ~\ref{assumption2}, with $x(0) \in \mathcal{X}$, system~\eqref{eq:model-dynamics} is infinitesimally contractive and the state is bounded in the range $[-1,1]$, if the following constraints are satisfied
    \begin{equation}
			\tau\geq 0, \quad \mu_{\infty}(A^T \otimes S) \geq 0; \quad x(0) \in \mathcal{X} \subseteq[-1,1]^{N \times G}
    \end{equation}
    \begin{equation}
        \begin{split}
              \text{with the} & \text{ contraction rate} \\ c =& \mu_{\infty}( \diag(\tau_{|}) + (A^T \otimes \tilde{S})+\diag(f_{|}))
        \end{split}
    \end{equation}
\end{theorem}
\begin{proof}
    The proof is available in~\cite{marino2024lgtc}.
\end{proof}
\begin{remark}
In practice, we replace $A$ with $A^TA$ and used the left normalized adjacency matrix for $S$ to guarantee the condition $ \mu_{\infty}(A^T \otimes S) \geq 0$ and  easily define $||\bar{S}||_{\infty}$,$||\tilde{S}||_{\infty}$
\end{remark}
To analyze the system \eqref{eq:model-dynamics},  we provide the following definition on the clustering of the state trajectories.
\begin{definition}
\label{def:cluster-trajectories}
    Given $N$ agent state trajectories $\bar{x}^c = [x^c_1, \dots, x^c_N]$, agents are grouped into clusters if they follow the same time-evolving state. Formally, agents $i$ and $j$ belong to the same cluster $C_m$ if and only if $ x^c_i = x^c_j, \quad \forall i, j \in C_m, \quad m \in \{1, \dots, M\}.$ That is, the cluster set is defined as $ C_m = \left\{ i,j \in \mathcal{V} \;\middle|\; x^c_i = x^c_j \right\}. $ The clustered team state can then be represented as $ \bar{x}^c = [x_{c_1} \otimes \bm{1}_{C_1}, \dots, x_{c_M} \otimes \bm{1}_{C_M} ]^T $ where $x_{c_m}$ is the representative state trajectory for cluster $C_m$.
\end{definition} 
We also provide an assumption on the structure of $S$.
\begin{assumption}
\label{assumption3}
    The graph representation matrix, $S$, has constant row sum in the block matrices containing clustering agents.
\end{assumption}
\begin{remark}
    Assumption \ref{assumption3} dictates a specific graph topology for the clusters. This is naturally satisfied when the similarity of $f_i$
  between connected agents defines the graph link weights. Moreover, in practice, even if this condition is relaxed, the agents will still converge to their respective clusters. However, the equilibrium states within each cluster will generally differ, with the differences being bounded and proportional to the variations in row sums among agents within the same cluster. 
\end{remark}
We state our main result.
\begin{theorem}
\label{theorem-2}
The team state trajectories $x$ of the distributed system in \eqref{eq:model-dynamics} converge exponentially to $M$ cluster trajectory states as given by Definition~\ref{def:cluster-trajectories}.
\end{theorem}

\begin{proof}
To prove the statement of the theorem, we need to show that the agents have the same particular state trajectories if they select the same dynamics (therefore if they belong to the same sub-team) and that this particular trajectory is a global attractor for the agents in the sub-team. \\  
Let us define the $f_{mi}$ dynamics corresponding to the $m$-th consumer selected by the agent $i$ as 
\begin{equation*}
f_{mi} = \rho(W \Phi_m + b).
\end{equation*}
Vector $\bar{f}_M = 1_N \otimes [f_{m1}, \dots, f_{mN}]^T$ consists of the consumer resource dynamics. Therefore the equilibrium trajectories are defined by:

\begin{equation}
    \dot{x}^c_i = - \tau \circ x^c_{i} + f_{mi} \circ (B_f-x^c_{i}) - S_i\bar{x}^cA
\end{equation}
with $S_i$ being the $i$-th row of $S$. The ReLU $= max(\cdot, 0)$ is analyzed in the two cases of the max operator for each element $q$ of the state $x^c_{i}$. Therefore, there exist two cases:
\begin{equation*}
\begin{cases}
    \begin{aligned}
        \dot{x}^c_{iq} &= -\tau_q x^c_i+f_{miq}(B_{fq}-x^c_{iq}) - S_i\bar{x}^c_qA_q\\
        \dot{x}^c_{iq} &= -\tau_q x^c_i-S_i\bar{x}^c_qA_q
    \end{aligned}
\end{cases}.
\end{equation*}
By selecting agent $j$ in the same cluster of $i$, the difference of their dynamics ($\dot{x}^c_{iq} - \dot{x}^c_{jq}$) is equal to
\begin{equation}
\label{eq:diff_agents}
\dot{x}^c_{iq} - \dot{x}^c_{jq} = S_j\bar{x}^c_qA_q - S_i\bar{x}^c_qA_q
\end{equation}
because they share the same dynamics $f_m$. The difference in eq.~\eqref{eq:diff_agents} is equal to zero only if $ S_j\bar{x}^c_qA_q = S_i\bar{x}^c_qA_q$. This condition under the assumption~\ref{assumption1} as demonstrated in~\cite{xia2011clustering}. Given the Lipschitz constant $L$ and $L_{S}$ of the state transition dynamics with respect to $f$ and the graph topology, the state trajectories converge exponentially to the trajectory $\bar{x}^c$ driven by the demanding dynamics $\bar{f}_M$ and $S_{M}$ graph topology satisfying the constant row sum condition, as the system is contracting, that is,
\begin{equation*}
\begin{aligned}
        || x(t) - \bar{x}^c(t) ||_{\infty} &< e^{-ct} ||x(0) - \bar{x}^c(0)  ||_{\infty} \\ & + \frac{L}{c}(1-e^{-ct})\sup_{\tau \in [0,t]}||f(\tau) - \bar{f}_{M}(\tau)||_{\infty}
        \\ &+ \frac{L_S}{c}(1-e^{-ct})\sup_{\tau \in [0,t]}|| S(\tau) - S_M(\tau) ||_{\infty}
\end{aligned}
\end{equation*}
with $c$ being the contraction rate defined by the Theorem \ref{diss-lgtc}.
\end{proof} 
A final MLP is used at the end of the architecture to process the system's output, $y$. 

\subsection{Policy Training}

The trainable parameters for the policy and value functions are denoted as \( \theta \) and \( \phi \), respectively. The agent's policy distribution is modeled as a normal distribution \( \pi(\theta) = \mathcal{N}(\mu(\theta), \sigma(\theta)) \), where the mean \( \mu(\theta) \) and standard deviation \( \sigma(\theta) \) are generated by the proposed architecture.

An independent Proximal Policy Optimization (IPPO) strategy is employed to learn the decentralized policy and value function for each agent. The policy and value losses are defined as follows:

\begin{equation}
    J_{\pi_i}(\theta, \theta_k) =  \min\left(
    \frac{\pi_{i}(z_i, \theta)}{\pi_{i}(z_i, \theta_k)}  A_i^{\pi_i(\cdot, \theta_k)}, \;\;
    g(\epsilon, A_i^{\pi_i(\cdot, \theta_k)})
    \right)
\end{equation}

\begin{equation}
    J_{V_i}(\phi) =  \left(V_i(\phi) - \left(A_i^{\pi_i(\cdot, \theta_k)} + V_i(\phi_k)\right)\right)^2
\end{equation}

Here, \( \epsilon > 0 \) is a small constant to promote safe exploration, and \( A_i^{\pi_i(\cdot, \theta_k)} \) is the advantage function, estimated using Generalized Advantage Estimation (GAE) \cite{schulman2015high}. To satisfy the contractivity conditions outlined in Theorem~\ref{diss-lgtc}, we added two regularization terms, \( \Pi_V(\phi) \) and \( \Pi_\pi(\theta) \), added to the value and policy losses, respectively, as follows:

\begin{equation}
    \begin{aligned}
    J_V &= \mathbb{E}_{ (a,z) \sim \mathcal{D}} \Bigg[ \sum_{i=1}^{N} J_{V_i} + \alpha \Pi_V \Bigg] \\ J_\pi &= \mathbb{E}_{ (a,z) \sim \mathcal{D}} \Bigg[ \sum_{i=1}^{N} J_{\pi_i} + \alpha \Pi_\pi \Bigg]
    \end{aligned}
\end{equation}

Where the team joint actions and the agents' observations $(a,z)$ are randomly sampled from the replay buffer $\mathcal{D}$ and ${\alpha>0}$ is a coefficient to softly enforce the contractivity constraints. The regularization terms are defined as:
\begin{equation}
    \Pi = \text{softplus}(c) + \text{softplus}(\tau).
\end{equation}

\section{EVALUATIONS}
\label{sec:evaluations}

\begin{figure}[t]
    \centering
    \includegraphics[width=1.\linewidth, height=0.45\linewidth]{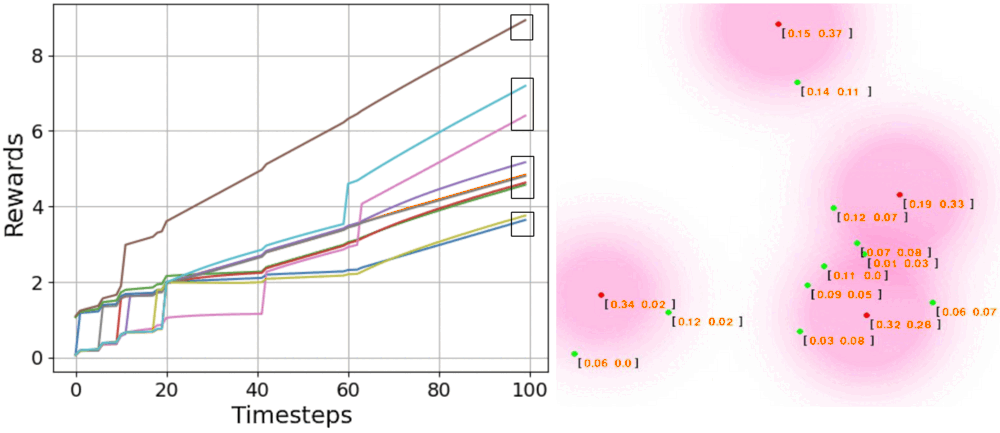}
    \caption{The agents' rewards cluster as the agents (green dots) spatially cluster to satisfy the multi-resource demand (red dots). The pink circle represents the interaction zone weighted by the demand intensity}
    \label{fig:reward_clustering}
\end{figure}

\begin{figure}[t]
    \centering
    \includegraphics[width=\linewidth]{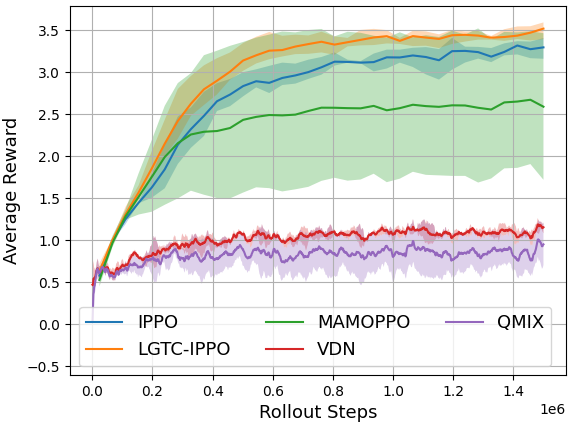}
    \caption{Mean (bold lines) and standard deviation (shaded area) of accumulated rewards over four training runs with random seeds, reported for our approach and four state-of-the-art MARL methods.}
    \label{fig:algorithms-comparison}
\end{figure}
\begin{figure}[t]
    \centering    \includegraphics[width=\linewidth, height=0.43\linewidth]{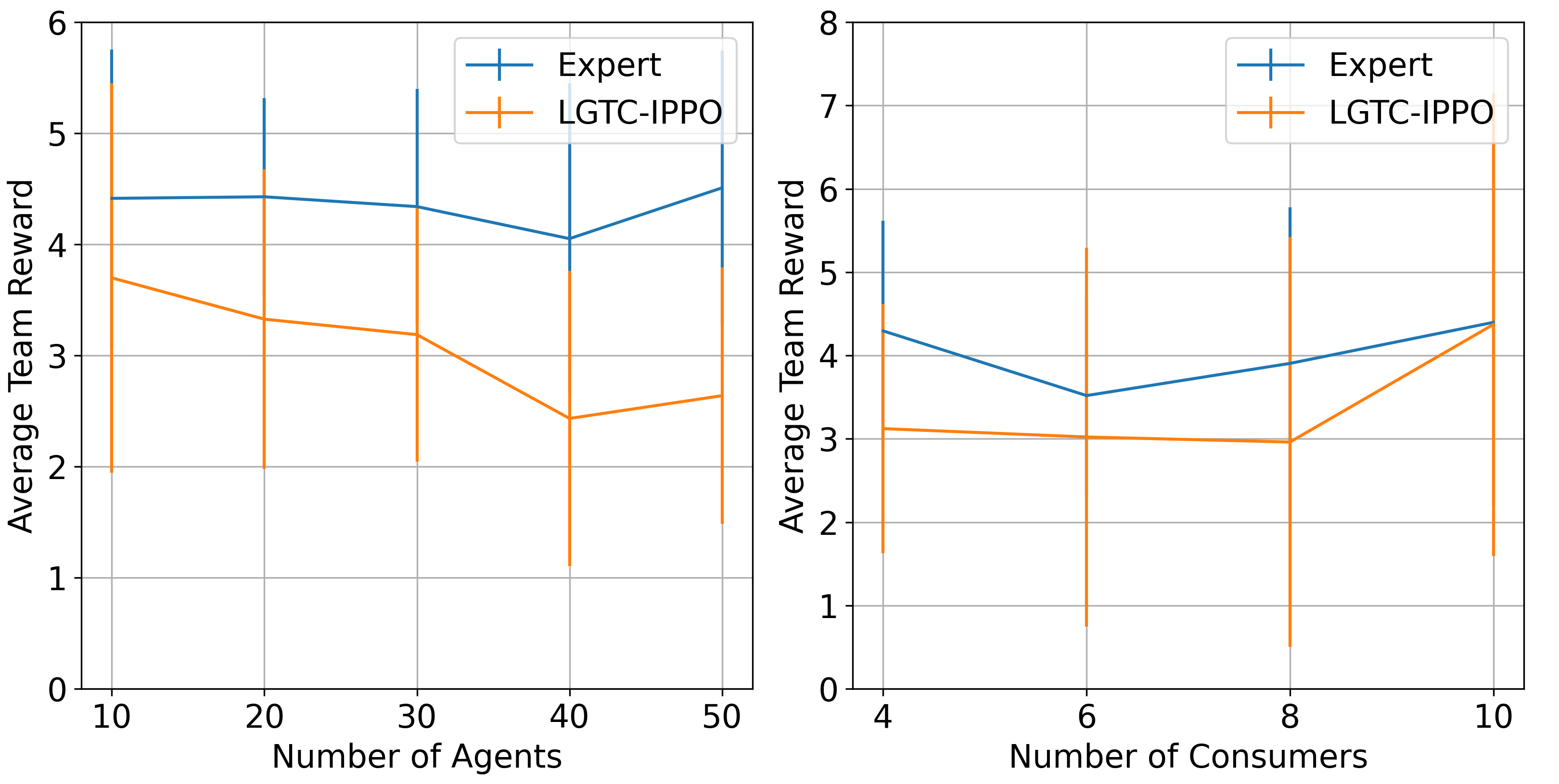}
    \caption{Mean and standard deviation of accumulated rewards over a variable number of agents and consumers, reported for a centralized expert and our approach (LGTC-IPPO).}
    \label{fig:reward_agents}
\end{figure}

\begin{table}[t]
\begin{tabular}{lll}
Parameters                      & MOMAPPO/IPPO/LGTC-IPPO & VDN/QMIX \\ \hline
$\gamma$                        &  0.99                  &  0.99        \\
$\lambda_{GAE}$                 &  0.95                  & -        \\
batch size                      &  64                    & 256         \\
epochs                          &  4                     & 1           \\
rollout/Replay buffer           &  6144                  & 20000         \\
entropy coefficient             &  -                     & 0.01         \\
$\epsilon$                        &  0.2                   & -         \\
max gradient norm               &  0.5                  & 0.5           \\
polyak                     &  -                    & 0.005        \\
G                               &  64                    & 64         \\
F                                &  64                      & 64        \\
final MLP                       &  [256,256]             & [256,256]        \\
MIX                             &  -                    & [32,32]        \\
adam learning rate              &  \num{2e-4}                  & \num{2e-4}           
\end{tabular}
\caption{Training parameters}
\label{tab:training_parameters}
\end{table}

We trained the proposed method in a simulated environment with $10$ agents and $4$ consumer locations, each requiring $2$ resources of different types. The agents are modeled as single integrators with velocity inputs. The resource types, the resource quantities and the agent's initial locations were randomized at each rollout reset. The PyTorch implementation of our algorithm is available at the following link\footnote{https://gitlab.inria.fr/amarino/multi-robot-cluster-assignment}. All training and evaluations were conducted on a machine running Ubuntu 22.04. with Intel Core i7-9750H @ 2.60GHz CPU, Nvidia RTX 2080Ti and 32G RAM. The agents are modeled as single integrator dynamics running at $20Hz$. At the same rate, the agents communicate with their neighbours and generate actions from the learned policy. We normalized the agent positions between $[-1,1]$ and considered a communication range of $C=1.$ in the normalized space.

We want to demonstrate the existence of reward clusters that arise from correct robots' assignments. Figure~\ref{fig:reward_clustering} presents an assignment case along with the agent's reward, defined as in Section~\ref{sec:reward-shaping}. As shown, the rewards cluster into similar values for different groups that emerge to complete the task. 

In the following, we present the results using the cumulative reward. This choice allows us to report the result in a compact way as the reward function aggregates multiple performance indexes e.g., the collision avoidance and the quantity of demand satisfied.
We compared our method with Value-Decomposition Networks (VDN) and QMIX to analyze the performance of a decentralized value function versus a global value function. Since the original implementations of VDN and QMIX were designed for discrete action spaces, we adapted these methods by incorporating a soft actor-critic (SAC)~\cite{haarnoja2018soft} strategy. SAC is an off-policy algorithm that uses a soft $Q$ function to maximize policy entropy, aligning with the off-policy nature of VDN and QMIX.

Specifically, denoting the soft Q function for the $i$-th agent as $sQ_i$, the soft Q function updates for VDN and QMIX are computed as follows:

\begin{equation*}
    \begin{aligned}
        J_{QVDN}(\phi) & = \mathbb{E}_{ (a,z) \sim \mathcal{D}} \Big[ \Big( \sum_{i=1}^{N} Q_i(z_i, a_i, \phi) - \mathrm{rw}(s, a) \\
        & \quad - \gamma \mathbb{E}_{ z_i' \sim s_i'} \big[ \mathbb{E}_{ a_i' \sim \pi(\cdot, \theta)} \big[ \sum_{i=1}^{N} sQ_i(z_i', a_i', \phi) \big] \big] \Big)^2 \Big], \\
        J_{QMIX}(\phi) & = \mathbb{E}_{ (a,z) \sim \mathcal{D}} \Big[ \Big( \text{MIX}(Q_i(z_i, a_i, \phi)) - \mathrm{rw}(s, a) \\
        & \quad - \gamma \mathbb{E}_{ z_i' \sim s'} \big[ \mathbb{E}_{ a_i' \sim \pi(\cdot, \theta)} \big[ \text{MIX}(sQ_i(z_i', a_i', \phi)) \big] \big] \Big)^2 \Big].
    \end{aligned}
\end{equation*}

Here, MIX represents a neural network that combines the individual agent Q functions into a global Q function.

We also benchmark our method against a vanilla IPPO variant, which excludes both the cluster consensus mechanism and loss regularization. For these three methods, we used a policy neural network architecture similar to the proposed method but excluded the dynamic cluster consensus dynamics. Instead, we retained only the consumer feature selection, input features, and agent positions, represented as:
\begin{equation*}
    y = \big[ (\Xi \circ \Phi)_{\argmax_{m \in \{1, \dots, M \}} \Xi_{im}}, \Delta, p \big].
\end{equation*}

However, for these algorithms, the attention $\Xi$ depends on the current input.Additionally, we compared our approach with Multi-Agent Multi-Objective PPO (MAMOPPO)\cite{felten2024momaland}, a leading method for addressing multi-objective multi-agent problems. MAMOPPO employs a weighted sum of resource features to predict both the value function and the policy. In our comparison, we adapted our architecture to predict a vector-valued function, where each element corresponds to the expected value associated with a specific demand location. The weighted sum of this vector is then trained to approximate the expected scalar reward defined in Section \ref{sec:reward-shaping}. The Table~\ref{tab:training_parameters} summarizes the training parameters.

Figure~\ref{fig:algorithms-comparison} compares the performance of the different algorithms across 10 random seeds. We named our solution LGTC-IPPO. Our solution achieves the highest mean reward with a low standard deviation of approximately 0.2 by the end of training. As expected, VDN and QMIX quickly stop improving due to the local nature of the rewards. MAMOPPO and IPPO perform closer to our approach, though MAMOPPO exhibits high variance. This is due to the weighted average of consumer features which can sometimes provide a good estimate of the final rewards but fails in the general case. IPPO, which shares the same neural network as our method, performs worse due to its lack of coordination in consumer selection, as it does not incorporate cluster consensus. 

Additionally, we compare our solution with an expert centralized approach while varying the number of agents and resources. We made 10 trials for each case. The centralized controller employs the optimization in Eq.\eqref{eq:QP} to assign the optimal resource to each agent and uses a proportional controller with centralized obstacle avoidance to guide agents toward their assigned locations while preventing collisions. Figure~\ref{fig:reward_agents} presents the average team reward for our proposed algorithm (LGTC-IPPO) compared to an expert centralized solution, as the number of agents (left) and the number of consumers (right) vary. The expert consistently outperforms LGTC-IPPO, maintaining a higher reward across different scenarios, as expected, since it uses perfect information about the environment. As the number of agents increases, both methods show a slight decline in performance, but LGTC-IPPO experiences a more pronounced drop. While the expert remains relatively stable around $4–4.5$, LGTC-IPPO's reward decreases significantly, falling below $3$ when the number of agents exceeds $30$. This suggests that handling conflicts between resource allocation and collision avoidance becomes increasingly challenging for the decentralized approach in crowded environments. Since these two objectives often conflict, solutions that prioritize one over the other can lead to local minima and inefficient interactions. Training a policy on larger teams could help mitigate this issue. On the other hand,  varying the number of consumers, the expert maintains a relatively stable reward, while LGTC-IPPO remains consistently lower but follows a similar pattern. Interestingly, for a high number of consumers ($10$), LGTC-IPPO reaches the expert’s performance, albeit with high variance. This suggests that when more consumers are available, the decentralized approach has more opportunities to make effective decisions, reducing the negative impact of local information limitations.   

\section{EXPERIMENTS}
\begin{figure*}[t]
    
    \begin{subfigure}[b]{0.33\textwidth}
    \includegraphics[width=\textwidth,height=0.17\textheight]{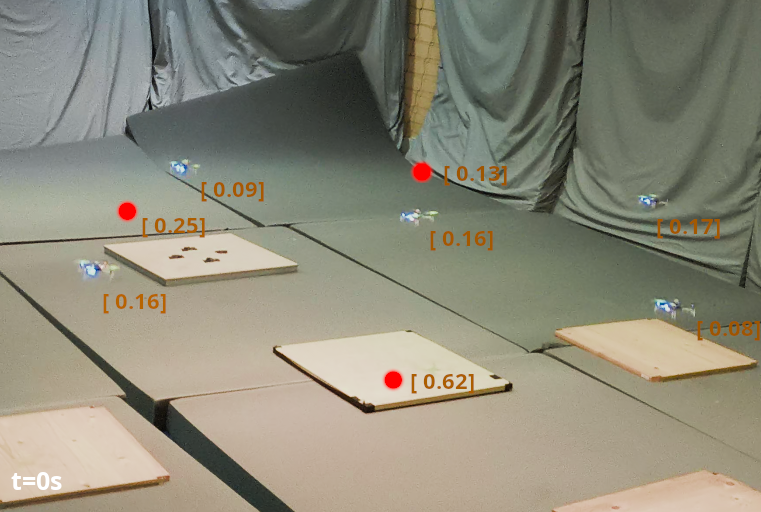}
    \end{subfigure}
    \begin{subfigure}[b]{0.33\textwidth}
        \includegraphics[width=\textwidth, height=0.17\textheight]{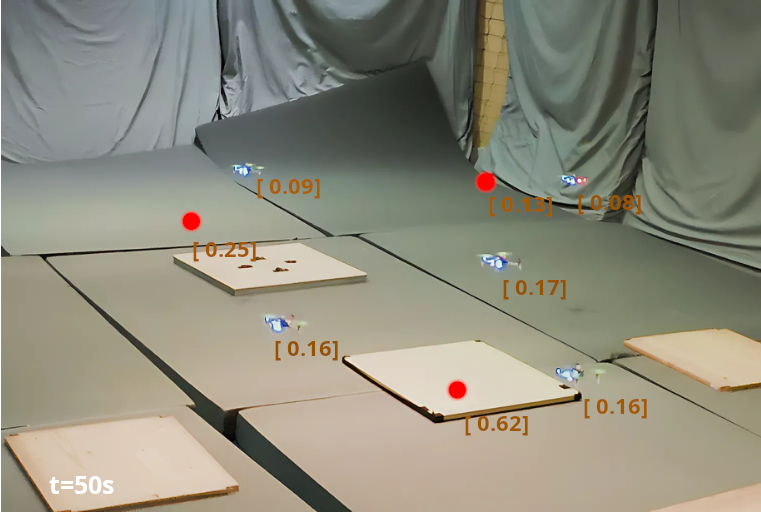}
    \end{subfigure}
    \begin{subfigure}[b]{0.33\textwidth}
        \includegraphics[width=\textwidth, height=0.17\textheight]{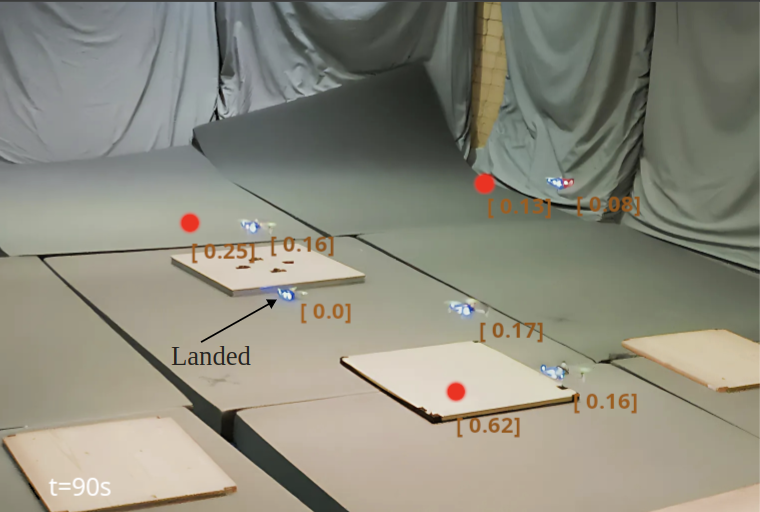}
    \end{subfigure}
    \caption{Dynamic agent reallocation with discharging resources. $5$ drones relocate in the space to interact with $3$ task locations (red dots). The battery level required by each task and available on each drone is depicted in orange. At $50$~s, one drone discharges and lands on the ground, therefore a drone of another subteam relocates to the task left uncovered. }
    \label{fig:discharge-scenarios}   
\end{figure*}
\begin{figure}
    \centering    \includegraphics[width=0.9\linewidth]{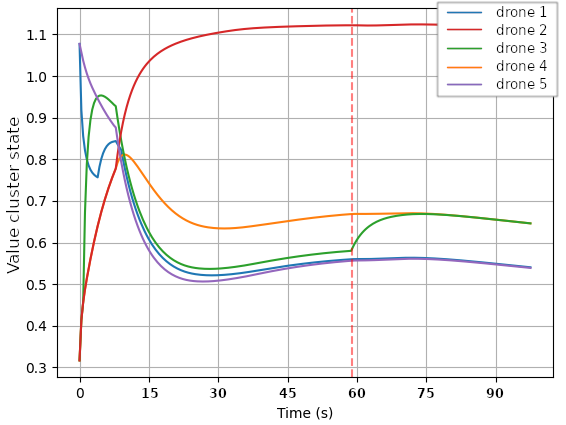}
    \caption{Dynamic state cluster consensus for the resource discharging case. At the time $50$~s, the resource available to drone~$3$ discharges and drone~$4$ changes cluster to cover the consumer unsatisfied.}
    \label{fig:dynamic_state_clustering}
\end{figure}
We evaluated the policy using $5$ Crazyflie drones, each loaded with up to $3$ resources of the two different types described in the Section II (persistent and instantaneous). The policy was executed on a central workstation, which transmitted velocity commands to the drones via a radio dongle. The hardware implementation also incorporated a local Control Barrier Function (CBF) safety filter. The CBF tracks the generated velocity command and adjusts it as necessary to consider a safety distance among the drones. This helps in mitigating the impact of the airflow disturbances caused by other drones on the dynamics of an individual drone. Videos of the experiments are provided in the supplementary materials. Resources requiring instantaneous depletion are colored in green, while persistent resources are displayed in orange. The physical experiments validated the resource assignment performance observed in the previous section.

To further analyze the system, we focused on a scenario where each drone was assigned a single 
persistent resource. This resource represented a dischargeable entity critical to the drone's operation, such as battery level. This setup allowed us to study the equilibrium changes in the dynamic cluster consensus process. Key moments of this scenario are shown in Figure~\ref{fig:discharge-scenarios}. Initially, after takeoff, the drones allocate their resources to target locations to minimize the remaining demand. At $55$~s, we manually reduce the resource level of one drone to simulate the battery discharge process. This causes the drone to land, triggering a reallocation process. To compensate for the absence of the landed drone, another drone is assigned to a new target location, completing the new allocation by approximately $90$~s. 
The value function dynamic states of the $5$ drones, shown in Figure~\ref{fig:dynamic_state_clustering} as a normalized projection on the unitary vector, illustrate the system's response. When the resource (battery) of one drone was depleted, a new set of clustering equilibria emerged, forming clusters that corresponded to the updated resource allocation solution. 
\section{CONCLUSION} 
In this work, we presented LGTC-IPPO, a decentralized reinforcement learning approach for multi-agent multi-resource allocation, incorporating cluster consensus to enhance coordination. By extending IPPO with localized consensus mechanisms, our method improves task allocation efficiency while effectively handling the challenges posed by group rewards. Through extensive evaluations, we demonstrated that LGTC-IPPO outperforms standard MARL baselines in terms of reward stability and coordination. Additionally, our results show that clustering equilibria facilitate agent reallocation when required by changes in consumer demand or team dynamics. Overall, the results show that while centralized approaches still have an advantage in scenarios where full information is available, LGTC-IPPO offers an alternative for environments where such global coordination is either infeasible or too costly.

For the future, we plan to refine the theoretical aspects of dynamic cluster consensus, focusing on its convergence properties in the face of variable communication topologies and stochastic disturbances. Moreover, exploring structured communication protocols and hierarchical planning strategies could further enhance the scalability of our framework. Ultimately, our goal is to extend the applicability of LGTC-IPPO to more complex, real-world scenarios, advancing the state of the art in decentralized multi-agent coordination.

\bibliographystyle{IEEEtran} 
\bibliography{IEEEabrv,bibliography}

\end{document}